\documentclass{KERauth}
\usepackage{algorithm}
\usepackage[noend]{algorithmic}
\usepackage{float}
\usepackage{subcaption}
\usepackage{wrapfig}
\usepackage{graphicx}
\usepackage{float}
\usepackage{color}
\usepackage[hidelinks]{hyperref}
\usepackage{times}
\usepackage{natbib}
\usepackage{dirtytalk}
\usepackage{booktabs}

\newtheorem{deff}{Definition} 

\newtheorem{proposition}{Proposition}

\def\statespace {{\cal S}}
\def\actionspace {{\cal A}}
\def\transitionmodel {{\cal T}}
\def\Reward {{\cal R}}

\begin{document}
\KER{1}{999}{00}{0}{2004}{S000000000000000}
\runningheads{A. Rosenfeld, M. Cohen, M. E. Taylor and S. Kraus}{Leveraging human knowledge in tabular reinforcement learning}

\title{Leveraging human knowledge in tabular reinforcement learning: A study of human subjects}

\author{ARIEL ROSENFELD\affilnum{1}, 
MOSHE COHEN\affilnum{2},
MATTHEW E. TAYLOR\affilnum{3}, and
SARIT KRAUS\affilnum{2}
}

\address{\affilnum{1} Department of Computer Science and Applied Mathematics, Weizmann Institute of Science, Rehovot, Israel. \\
\email{arielros1@gmail.com} \\
\affilnum{2} Department of Computer Science, Bar-Ilan University, Ramat-Gan, Israel. \\
\affilnum{3} Department of Computer Science, Washington State University, Pullman, Washington, USA. \\
}


\begin{abstract}
Reinforcement Learning (RL) can be extremely effective in solving complex, real-world problems. However, injecting human knowledge into an RL agent may require extensive effort and expertise on the human designer's part. To date, human factors are generally not considered in the development and evaluation of possible RL approaches. In this article, we set out to investigate how different methods for injecting human knowledge are applied, \textit{in practice}, by human designers of varying levels of knowledge and skill. We perform the first empirical evaluation of several methods, including a newly proposed method named \textbf{SASS} which is based on the notion of similarities in the agent's state-action space. Through this human study, consisting of 51 human participants, we shed new light on the human factors that play a key role in RL. We find that the classical \textbf{reward shaping} technique seems to be the most natural method for most designers, both expert and non-expert, to speed up RL. However, we further find that our proposed method \textsc{SASS} can be effectively and efficiently combined with reward shaping, and provides a beneficial alternative to using only a single speedup method with minimal human designer effort overhead. 
\end{abstract}


\section{Introduction}

\noindent Reinforcement Learning \citep{sutton1998reinforcement} (RL) has had many successes solving complex, real-world problems. However, unlike supervised machine learning, there is no standard framework for non-experts to easily try out different methods (e.g., Weka \citep{witten2016data}), which may pose a barrier to wider adoption of RL methods. While many frameworks exist, such as RL-Glue \citep{rl-glue}, RLPy \citep{RLPy}, PyBrain \citep{schaul2010}, OpenAI-Gym \citep{openai-gym} and others, they all assume some and sometimes even a substantial amount of RL knowledge. Substantial effort is required to add new tasks or instantiate different techniques within these frameworks. Another barrier to wider adoption of RL methods, which is the focus of this article, is the fact that injecting human knowledge (which can significantly improve the speed of learning) can be difficult for a human designer. 

When designing an RL agent, human designers must face the question about how much human knowledge to inject into the system and, more importantly, which approach to use for injecting the desired knowledge. From the \textit{AI research perspective}, the more an agent can learn autonomously, the more interesting and beneficial the agent will be. From the \textit{engineering or more practical perspective}, more human input is desirable as it can help improve the agent's learning as well as the speed at which the agent learns. However, this knowledge is only useful as long as it is practical for it to be gathered and leveraged by the human designer. 
In order for RL methods to move beyond requiring developers to fully understand the \say{black arts} of generalization, approximation and biasing, it is critical that the community better understand if and how both expert and non-expert humans can provide useful information for an RL agent. This article \textit{takes the problem to the field} and focuses on human designers who have a background in AI and coding, but varying experience in RL. 

The baseline approach in this study is to allow no generalization: an agent's interactions with its environment will immediately affect only its current state in a tabular representation. We will compare this baseline with two widely common speedup approaches: \textit{function approximation} \citep{busoniu2010reinforcement} (FA), and \textit{Reward Shaping} (RS) \citep{mataric1994reward}. We further propose and evaluate a novel approach which we name \textsc{SASS}, which stands for \textit{State Action Similarity Solutions}, which relies on hand-coded state-action similarity functions. We test the three speedup approaches in a first-of-its-kind human study consisting of three experts (highly experienced programmers with an RL background, but not co-authors of this paper) and 48 non-expert computer science students.\footnote{All experiments were authorized by the corresponding institutional review board.}
To that end, three RL tasks of varying complexities are considered: 
\begin{enumerate}
\item The ``toy" task of simple robotic soccer \citep{littman1994markov}, providing a basic setting for the evaluation.
  \item A large grid-world task named Pursuit \citep{benda1985optimal}, investigating the three speedup methods on a moderately challenging task. 
\item The popular game of Mario \citep{karakovskiy2012mario}, exemplifying the complexities of instantiating speedup approaches in complex tasks.
\end{enumerate}

Through this human study, we find that the \textsc{RS} technique is the most natural method for most designers, both expert and non-expert, to speed up RL across the three tasks. However, we further find that the newly proposed \textsc{SASS} method can be effectively and efficiently combined with \textsc{RS}, providing an additional speedup in most cases. 

This article argues that in order to bring about a wider adoption of RL techniques, specifically of generalization techniques, it is essential to both investigate and develop RL techniques appropriate for both expert and non-expert designers.  We hope that this study will encourage other researchers to invest an increased effort in the human factors behind RL and investigate their RL solutions in human studies.

The remainder of the article is organized as follows: In Section \ref{sec:RW} we review some preliminaries on RL and survey recent related work. In Section \ref{sec:similarities} we present the $QS$-learning algorithm that incorporates similarities within the basic $Q$-learning framework and discuss its theoretical foundations. We further propose three notions of state-action similarities and discuss how these similarities can be defined.  In Section \ref{sec:eval} we present an extensive empirical evaluation of the three tested approaches in three RL tasks. Finally, in Section \ref{sec:conclusions} we provide a summary and list future directions for this line of work.





\section{Preliminaries and Background}
\label{sec:RW}

An RL agent generally learns how to interact with an unfamiliar environment \citep{sutton1998reinforcement}. We define the RL task  using  the standard notation of a \textit{Markov Decision Process} (MDP). 
An MDP is defined as $\langle \statespace,\actionspace, \transitionmodel,\Reward,\gamma \rangle$ where:
	\begin{itemize}
	\item $\statespace$ is the state-space; 
	\item $\actionspace$ is the action-space;
	\item $\transitionmodel: \statespace\times \actionspace\times \statespace \rightarrow [0,1]$ defines the transition function, where  $\transitionmodel(s, a, s^\prime)$ is the probability of making a transition from state $s$ to state $s^\prime$ using action $a$; 
	\item $\Reward: \statespace\times \actionspace\times\statespace \rightarrow \mathbb{R}$ is the reward function; and
	\item $\gamma \in [0,1]$ is the discount factor, which represents the significance of future rewards compared to present rewards. 
	\end{itemize}
    
We assume $\transitionmodel$ and $\Reward$ are initially unknown to the agent. In discrete time tasks, the agent interacts in a sequence of time steps. At each time step, the agent observes its state $s\in\statespace$ 
and is required to select an action $a\in\actionspace$. The agent then  arrives at a new state $s^\prime$ according to the unknown $\transitionmodel$ function and receives a reward $r$ according to the unknown $\Reward$ function. We define an agent's \textit{experience} as a tuple $\langle s, a, r, s^\prime\rangle$ where action $a$ is taken in state $s$, resulting in reward $r$ and a transition to the next state, $s^\prime$.
The agent's objective is to maximize the accumulated discounted rewards throughout its lifetime. Namely, the agent seeks to find a policy $\pi:\statespace\mapsto\actionspace$ that maximizes the expected total discounted reward (i.e., expected return) from following it.

Temporal difference RL algorithms such as 
$Q$-learning \citep{watkins1989learning} approximate an action-value function $Q:\statespace\times\actionspace\mapsto\mathbb{R}$, mapping state-action pairs to the expected real-valued discounted return. $Q$-learning updates the $Q$-value estimation according to the temporal difference update rule 
\[
Q(s,a) = Q(s,a) + \alpha (r+ \gamma {max}_{a^\prime} Q(s^\prime,a^\prime) - Q(s,a))
\]
\noindent where $\alpha$ is the learning rate.
If both $\statespace$ and $\actionspace$ are finite sets, the $Q$ function can be easily represented in a table, namely in an $|\statespace|\times|\actionspace|$ matrix, where each state-action pair is saved along with its discounted return estimation. 
In this case, the convergence of $Q$-learning 
has been proven in the past (under standard assumptions). See \cite{sutton1998reinforcement} for more details. 
In this study, we focus on the $Q$-learning algorithm with a tabular representation of the $Q$ function. This scheme is, perhaps, the most basic and commonly applied in RL tasks and allows us to control for many of the  confounding factors in human experiments (e.g., implementation complexity). 

RL can often suffer from slow learning speeds. To address this problem, designers infuse human-generated, domain-specific knowledge into the agent's learning process in different ways, enabling better generalization across small numbers of samples. 
Another interpretation for this approach is allowing the agent to better understand and predict what a \say{human agent} would do or conclude in a given setting and leverage this prediction to make better decisions on its own (see \cite{rosenfeld2018predicting}).  

Perhaps the most prominent method for leveraging human knowledge to speedup RL is \textit{Function Approximation} \citep{busoniu2010reinforcement} (FA).
The \textsc{FA} approach focuses on mitigating the costs associated with maintaining and manipulating the value of every state-action pair, as in the usual tabular representation case. Specifically, using \textsc{FA}, a designer needs to abstract the state-action space in a sophisticated manner such that the (presumed) similar states or state-action pairs will be updated together and dissimilar states or state-action pairs are not. This allows the RL learner to quickly generalize each of its experiences such that the value of more than a single state or state-action pair is updated simultaneously. 
\textsc{FA} is based on the premise that a human designer can recognize features or pattens in the environment by which one can determine a successful policy. Many successful RL applications have used highly engineered state features to bring about a successful learning performance (e.g., `the distance between the simulated robot soccer player with the ball to its closest opponent' and `the minimal angle with the vertex at the simulated robot soccer player with the ball between the closest teammate and any of the opponents' \citep{LNAI2005-keepaway}). 
With the recent successes of DeepRL \citep{mnih2015human}, convolutional neural networks were shown to successfully learn features directly from pixel-level representations. However, such features are not necessarily optimal. A significant amount of designer time is necessary to define the deep neural network's architecture, and a significant amount of data is required to learn the features.

Another popular approach of injecting human knowledge to an RL learner is \textit{Reward Shaping} (RS) \citep{mataric1994reward}. Reward shaping attempts to bias the RL learner's decision-making  by adding additional localized rewards that encourage a behavior consistent with some prior knowledge of the human designer. Specifically, instead of relying solely on the reward function $\Reward$, the agent considered an augmented reward signal $R(s,a,s')+F(s,a,s')$ where $F$ is the shaping reward function articulated by the human designer. The \textsc{RS} approach is inspired by Skinner's recognition of the effectiveness of training an animal by reinforcing successive approximations of the desired behavior \citep{skinner1958reinforcement}. While the use of \textsc{RS} may result in undesirable learned behavior in the general case (e.g., \citep{randlov1998learning}), if \textsc{RS} is applied carefully (e.g., using the  Potential Based Reward Shaping (PBSR) method \citep{ng1999policy}), one can guarantee that the resulting learned policy is unaltered and, in many cases, produce a significant speedup. 
Indeed, \textsc{RS} has been utilized by many successful RL applications, significantly speeding up the agent's learning process (e.g., 'encouraging simulated robotic soccer players to spread out across the field' and 'encouraging a simulated robotic soccer player to tackle the ball on defense' \citep{devlin2011multi}).

Another related line of research investigates providing direct biasing from non-expert humans, such as incorporating human-provided 
feedback \citep{AAMAS10-knox,peng2016need} or demonstrations \citep{brys2015reinforcement}. For example, one may ask a non-expert user to teleoperate the agent or provide online feedback for the agent's actions. These methods do not require significant technical abilities on the part of the human (e.g., programming is not needed). In this article, we consider a possibly complementary approach, leveraging \textit{technically-able human designers'} knowledge and technical abilities, either in terms of providing an abstraction, a reward shaping function or a similarity function, to improve the agent's performance. The investigated approaches can be integrated with direct biasing as well. We leave the examination of non-technical methods (e.g., direct biasing) and non-technical designers (e.g., designers who cannot program) for future work.


While the above (and other) methods of leveraging human knowledge to speed up RL learners have been thoroughly investigated with respect to their theoretical properties and empirical performance in various settings, their deployment often requires extensive engineering and expertise on the designers' part. To the best of our knowledge, designers' efforts and expertise have not been explicitly considered in past works (e.g., methods are not evaluated in terms of the amount of time a developer must invest in order to fine-tune parameters, select appropriate state representations, etc., and developers' experience and expertise are generally not considered). \textit{This is the first article to examine these two issues in practice.}

Our proposed method, \textsc{SASS}, is investigated theoretically and empirically in this study. \textsc{SASS} heavily relies on the notion of generalization through \textit{similarity}. This notion is also common in other techniques that allow the learning agent to provide predictions for unseen or infrequently visited states.
For instance, Texplore \citep{hester2013texplore} uses supervised learning techniques 
to generalize the effects of actions across different states. The assumption is that actions are likely to have similar effects across states. 
\cite{tamassia2016dynamic} suggest a different approach: dynamically selecting state-space abstraction by which different states that share the same abstraction features are considered similar.
\cite{sequeira2013associative} and \cite{girgin2007positive} have presented variations of this notion by identifying associations online between different states in order to define a state-space metric or equivalence relation. However, all of these methods assume that an expert RL designer is able to iteratively define and test the required similarities without explicit cost. This is not generally the case in practice.

Note that alternative updating approaches such as eligibility traces \citep{sutton1998reinforcement}, where multiple states can be updated based on time since visitation, are popular as well. For ease of analysis, this study does not directly address such methods, which are left for future work. 

\subsection*{QS-learning}\label{sec:QS}

In order to integrate our \textsc{SASS} generalization approach within the $Q$-learning framework, we adopt a previously introduced technique \citep{ribeiro1995attentional} where $Q$-learning is combined with a spreading function that ``spreads" the estimates of the $Q$ function in a given state to neighboring states, exploiting an assumed spatial smoothness of the state-space. Formally, given an experience $\langle s,a,r,s^\prime \rangle$ and a spreading function $\sigma:\statespace\times\statespace\mapsto[0,1]$ that captures how \textit{``close"} states $s$ and $\tilde{s}$ are in the environment, a new update rule is used:

\begin{equation}
 Q(\tilde{s},a) = Q(\tilde{s},a) + \alpha \sigma(s,\tilde{s})\delta
\label{eq:QSupdate} 
\end{equation} 
\noindent where $\delta$ is the temporal difference error term ($r+ \gamma max_{a^\prime} Q(s^\prime,a^\prime) - Q(s,a)$).
The update rule in Eq.\ \ref{eq:QSupdate} is applied to all states in the environment after each experience. 
The resulting variation is denoted as $QS$-learning ($S$ stands for spreading). 
This method was only tested with author-defined spreading functions in simple grid worlds.




Note that standard $Q$-learning is a special case of $QS$-learning by setting the function $\sigma$ to the Kronecker delta ($\delta_k(x,y) = 1$ if $x = y$, otherwise $\delta_k(x,y) = 0$). 
\begin{proposition}
$QS$-learning converges to the optimal policy given the standard condition for convergence of $Q$-learning and either: 1) $\sigma$ which is fixed in time; or 2) $\sigma$ that converges to the Kronecker delta over the state-action space at least as quickly as the learning rate $\alpha$ converges to zero.
\end{proposition}
\begin{proof}
The proposition is a combination of two proofs available in \citep{szepesvari1999unified} and  \citep{ribeiro1996q}.
Both were proven for the update rule of Eq.\ \ref{eq:QSupdate} without loss of generality, and therefore apply to the $QS$-learning update rule of Eq.\ \ref{eq:QSAupdate} (page \pageref{eq:QSAupdate}) as well.
\end{proof}





\section{The SASS Approach}\label{sec:similarities}

Our proposed method, \textsc{SASS}, leverages a human designer's \textit{constructivism} \citep{bruner1957going}, specifically \textit{personal construct psychology} \citep{kelly1955personal}. Constructivism is a well-established psychological theory where people make sense of the world (situations, people, etc.) by making use of constructs (or clusters), which are  perceptual categories used for evaluation by considering members of the same construct as \textit{similar}. It has been shown that people who have many different, possibly overlapping, and abstract constructs have greater flexibility in understanding the world and are usually more robust against inconsistent signals.  The \textsc{SASS} approach is inspired by constructivism, allowing a designer to define both complex as well as simplistic constructs of similar state-action pairs according to one's knowledge, abilities and beliefs, and refine them as more experience is gained. This approach is in contrast to more complex types of generalization (e.g., specifying the width of a tile, the number of tiles, and the number of tilings in a CMAC~\citep{Albus:1981:BBR:542806} or specifying the number of neurons, number of layers, and activation functions in a deep net). 
Specifically, in designing and testing an RL agent, the human designer himself learns the traits of the domain at hand by identifying patterns and domain-specific characteristics. To accommodate both prior knowledge and learned insights (which may change over time),
it is necessary to allow the designer to easily explore and refine different similarity hypotheses (i.e., constructs). 
For instance, a designer may have an initial belief that the state-action pair $s,a$ has the same expected return as some other state-action pair $\tilde{s},\tilde{a}$. Using \textsc{FA}, this can easily be captured by mapping both pairs into a single meta state-action pair. However, after gaining some experience in the domain, the designer refines his belief and presumes that the two pairs are merely \textit{similar} (they would have close expected returns if they were to be modeled separately). This difference can have a significant effect on both the learning efficiency and the resulting policy (which may be suboptimal).


In this study, we assume that the similarity function is defined and refined by a \textit{human designer} during the development of the RL agent as follows:

\begin{deff}\label{def:similarityFunction}
Let $\statespace$, $\actionspace$ be a state-space and an action-space, respectively.\\
A \textbf{similarity function} $\sigma:\statespace\times\actionspace\times\statespace\times\actionspace\mapsto[0,1]$ maps every two state-action pairs in $\statespace\times\actionspace$ to the degree to which we expect the two state-action pairs to have a similar expected return. 
$\sigma$ is considered \textbf{valid} if $\forall \langle s,a\rangle\in\statespace\times\actionspace$.  $\sigma(s,a,s,a)=1$.
\end{deff}

Similarity functions can be defined in multiple ways in order to capture various assumptions and insights about the state-action space. As shown in constructivism literature \citep{bruner1957going}, some people may use simplistic, crude similarities that allow quick (and usually, inaccurate) generalizing of knowledge across different settings. Others may use complex and sophisticated similarity functions that will allow a more fine-grained generalization.  Although people can easily identify similarities in real-life, they are often incapable of articulating sophisticated rules for defining such similarities. Therefore, in the following, we identify and discuss three notable similarity notions that were encountered repeatedly in our human study (Section \ref{sec:eval}), covering the majority of human-designed similarity functions in our tested domains.  
\begin{enumerate}
\item \textbf{Representational Similarity} from the tasks' state-action space. \textsc{FA} is perhaps the most popular example of the use of this technique. The function approximator (e.g., tile coding, neural networks, abstraction, etc.) approximates the $Q$-value 
and therefore implicitly forces a generalization over the feature space. A common method 
is using a factored state-space representation, where each state is represented by a vector of features that capture different characteristics of the state-space. Using such abstraction, one can define similarities using an index over the factored state-action (e.g., \citep{sequeira2013associative,brys2015reinforcement}).   
Defining representational similarities introduces the major engineering concern of choosing the right abstraction method or \textsc{FA} that would work well across the entire state-action space, while minimizing generalizing between dissimilar state-actions.
Representational similarity has repeatedly shown its benefit in real-world applications, but no one-size-fits-all method exists for efficiently representing the state-action space. See Figure \ref{fig:all} (a) for an illustration.

%
\item \textbf{Symmetry Similarity} seeks to 
consolidate state-action pairs that are identical or completely symmetrical in order to avoid redundancies. 
\cite{zinkevich2001symmetry} formalized the concept of symmetry in MDPs 
and proved that if such consolidation of symmetrical state-actions is performed accurately, then the optimal $Q$ function and the optimal policy are not altered. However, automatically identifying symmetries is computationally complex \citep{narayanamurthy2008hardness}, especially when the symmetry is only \textit{assumed}. For example, in the Pursuit domain, one may consider the 90$^{\circ}$, 180$^{\circ}$ and 270$^{\circ}$ transpositions of the state around its center (along with the direction of the action) as being similar (see Figure \ref{fig:all} (b)). However, as the predators do not know the prey's (potentially biased) policy, 
they can only assume such symmetry exists. 


\item \textbf{Transition Similarity} can be defined based on the idea of 
\textit{relative effects} of actions in different states. A relative effect is a change in the state's features caused by the execution of an action. Exploiting relative effects to speed up learning was proposed \citep{jong2007model,leffler2007efficient} in the context of model learning. For example, in the Mario domain, if Mario \textit{walks right} or \textit{runs right}, outcomes are assumed to be similar as both actions induce similar relative changes to the state (see Figure \ref{fig:all} (c)).
In environments with complex or non-obvious transition models, it can be difficult to intuit this type of similarity.
\end{enumerate}


\begin{figure}[tpb]%
\center
\includegraphics[width=\linewidth]{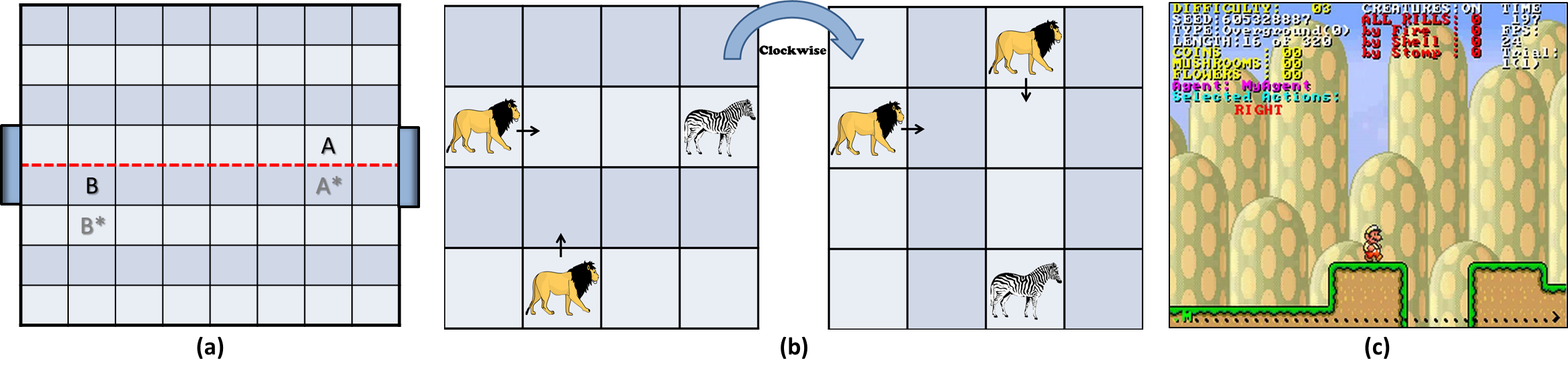}%
\caption{(a) Players in the simple robotic soccer task are A and B; the state in which the two players are moved one cell down (A* and B*) should be considered similar.
(b) Two (presumed) similar state-action pairs in the Pursuit domain. (c) A state in the Mario task where walking or running right are considered similar (i.e., falling into the gap).}%
\label{fig:all}%
\end{figure}


\subsection*{SASS in the $Q$-learning Framework}

We use the designer-provided similarity function $\sigma(s,a,s^\prime,a^\prime)$  
 instead of the spreading function needed by the $QS$-learning algorithm (as discussed in Section \ref{sec:QS}). 
In words, for each experience $\langle s, a, r, s^\prime \rangle$ that the agent encounters, depending on the similarity function $\sigma$, we potentially update more than a single $\langle s,a\rangle$ entry in the $Q$ table. Multiple updates, one for each entry $\langle\tilde{s},\tilde{a}\rangle$ for which $\sigma(s,a,\tilde{s},\tilde{a})>0$, are performed using 
 the following update:
\begin{equation}
Q(\tilde{s},\tilde{a}) = Q(\tilde{s},\tilde{a}) + \alpha \sigma(s,a,\tilde{s},\tilde{a})\delta
\label{eq:QSAupdate}
\end{equation}
\noindent which, as discussed in Section \ref{sec:QS}, does not compromise the theoretical guarantees of the unadorned $Q$-learning.

The update rule states that as a consequence of experiencing $\langle s,a,r,s^\prime\rangle$, an update is made to other pairs $\langle\tilde{s}, \tilde{a}\rangle$ \textit{as if} the real experience was actually $\langle\tilde{s},\tilde{a},r,s^\prime \rangle$ (discounted by the similarity function). 

In order to avoid a time complexity of $O(|S||A|)$ per step, $QS$-learning should be restricted to update state-action pairs for which the similarity is larger than $0$. In our experiments (see Section \ref{sec:eval}) we found only a minor increase in time-complexity for most human-provided similarity functions. 

In the interest of clarity, from this point forward we will use the term $QS$-learning using the above \textsc{Q}-learning-with-\textsc{SASS} interpretation. Namely, using a designer-defined similarity function $\sigma$ and the update rule of Eq. \ref{eq:QSAupdate}, we will modify the classic $QS$-learning algorithm yet keep its original name due to their inherent resemblance. See Algorithm \ref{alg:QS} for the QS-learning Algorithm as used in this study.

\begin{algorithm}[h]
\caption{ \label{alg:QS} $QS$-learning Algorithm}
\begin{algorithmic}
	\REQUIRE State-space $\statespace$, Action-space $\actionspace$, discount factor $\gamma$, learning rate $\alpha$, similarity function $\sigma$\\
	\STATE initialize $Q$ arbitrarily (e.g. $Q(s,a)=0$)\\
	\FOR{t=1,2,\ldots}
    	\STATE $s$ is initialized to the \emph{starting} state\\
		\REPEAT
			\STATE choose an action $a \in A(s)$ based on $Q(s,a)$ and an exploration strategy\\
            \STATE perform action $a$
            \STATE observe the new state $s^\prime$ and receive reward $r$\\
			\STATE calculate temporal difference error: $\delta \gets r + \gamma \cdot \max_{ a^\prime \in A} Q(s^\prime, a^\prime) - Q(s,a)  $ \\  
			\FOR{\textbf{each} $\tilde{s},\tilde{a}\in \statespace\times\actionspace$ such that $\sigma(s,a,\tilde{s},\tilde{a})>0$}
				\STATE $Q(\tilde{s},\tilde{a}) = Q(\tilde{s},\tilde{a}) + \alpha \sigma(s,a,\tilde{s},\tilde{a})\delta$\\
			\ENDFOR
            \STATE $s \gets s^\prime$\\
		\UNTIL{$s^\prime$ is a terminal state}
	\ENDFOR
\end{algorithmic}
\end{algorithm} 


\section{Evaluation}\label{sec:eval}

Our human subject study is comprised of three experimental settings: First, we examine the \textsc{SASS} approach against a baseline learner (i.e., no speedup method) and the \textsc{FA} approach in the \textit{simple robotic soccer} task with 16 non-expert developers. Through this experiment, which we will refer to as \textbf{Experiment 1}, we show the potential benefits of the \textsc{SASS} approach compared to \textsc{FA} given basic, classic reward shaping taken from previous works. Next, we evaluate all three speedup approaches (\textsc{FA}, \textsc{RS}, and \textsc{SASS}) along with a baseline learner using the \textit{Pursuit} and \textit{Mario} tasks. Through this experiment, which we will refer to as \textbf{Experiment 2}, we find that reward shaping provides the most natural approach of the three for most non-expert developers. However, the results further show that the combination of \textsc{RS} and \textsc{SASS} (as was tested in \textit{Experiment 1}) can bring about significant potential benefits with minimal overhead effort. Lastly, in \textbf{Experiment 3}, we evaluate all three tasks using three \textit{expert developers}. The results support our findings in Experiments 1 and 2, demonstrating high effectiveness for the combination of \textsc{RS} and \textsc{SASS} compared to the individual use of each approach.  

Throughout this section, we will use the following notations: a basic \textsc{Q}-learning agent is denoted \textsc{Q}, a \textsc{QS}-learning agent is denoted \textsc{QS}, a \textsc{Q}-learning agent that uses state-space abstraction is denoted \textsc{QA}, a \textsc{Q}-learning agent that uses reward shaping is denoted \textsc{QR} and an agent which combines reward shaping and similarities is denoted \textsc{QRS}. 

We first discuss the three domains we tested in this study followed by the three experiments.

\subsection{Evaluated Domains}\label{sec:domains}

\subsubsection{Simple Robotic Soccer}\label{sec:SRS}
Proposed in \citep{littman1994markov}, the task is performed on an $8\times 8$ grid world, defining the state-space $S$. Two simulated robotic players occupy distinct cells on the grid and can either move in one of the four cardinal directions or stay in place (5 actions each).  The simulated robots are designed to play a simplified version of soccer: At the beginning of each game, players are positioned according to Figure \ref{fig:all}(a) and possession of the ball is assigned to one of the players (either the learning agent or the fixed, hand-coded policy opponent\footnote{The opponent was given a hand-coded policy, similar to that used in the original paper, which instructs it to avoid colliding with the other player while it has the ball and attempts to score a goal. While defending, the agent chases its opponent and tries to steal the ball.}). During each turn, both players select their actions simultaneously 
and the actions are executed in random order. When the attacking player (the player with the ball) executes an action that would take it to a square occupied by the other player, possession of the ball goes to the defender (the player without the ball) and the move does not take place. A goal is scored when the player with the ball enters the other player's goal region. Once a goal is scored the game is won; the agent who scored receives 1 point, the other agent receives -1 point and the game is reset. The discount factor was set to 0.9, as in the original paper. 

We used a basic state-space representation, as done in \cite{martins2013heuristically}, a recent investigation of the game. 
A state $s$ is represented as a 5-tuple $\langle x_A$, $y_A$, $x_B$, $y_B$, $b\rangle$ where $x_i$ and $y_i$ indicate player \textit{i}'s position on the grid and $b\in\{A,B\}$ indicates which player has the ball.
The action-space is defined as a set of 5 actions as specified above. Overall, the state-action space consists of approximately 41,000 state-action pairs ($8^4 \ (\text{locations}) \times 2 \ (\text{ball-possession}) \times 5 \ (\text{actions})$).

\subsubsection{Pursuit}\label{sec:chase}
The Pursuit task (also known as Chase or Predator/Prey task) was proposed by \cite{benda1985optimal}. 
For our evaluation, we use the recently evaluated instantiation of Pursuit implemented in \cite{brys2014combining}. According to the authors' implementation, there are two predators ($Pred_1, Pred_2$) and one prey ($Prey$), each of which can move in one of the four cardinal directions as well as stay in place (5 actions each) on a $20\times 20$ grid world. The prey is caught when a predator moves onto the same grid cell as the prey. In that case, a reward of $1$ is given to the predator, $0$ otherwise.   
A state $s$ is represented as a 4-tuple $\langle\Delta_{x_1},\Delta_{y_1},\Delta_{x_2},\Delta_{y_2}\rangle$ where $\Delta_{x_i}$ ($\Delta_{y_i}$) is the difference between predator $i$'s x-index (y-index) and the prey's x-index (y-index). 
Overall, the state-action space consists of approximately 46 million state-action pairs ($39^4 ~(\text{differences}) \times 20 ~(\text{actions}))$. 

\subsubsection{Mario}\label{sec:Mario}

Super Mario Bros is a popular 2-D side-scrolling video game developed by the Nintendo Corporation. This popular game is often used for the evaluation of RL techniques \citep{karakovskiy2012mario}. In the game, the player's figure, Mario, seeks to rescue the princess while avoiding obstacles, fighting enemies and collecting coins. 
We use the recently evaluated formulation of the Mario task proposed by \cite{suay2016learning}. The authors use a 27-dimensional discrete state-variables representation of the state-space and model 12 actions that Mario can take. We refer the reader to the original paper for the complete description of the underlying MDP and parameters. 
Given the authors' abstraction of the state-space, the size of the state-action space is over \textit{100 billion}, although many of the possible states are never encountered in reality. For example, it is impossible to have Mario trapped by enemies from all directions at the same time. Due to the huge state-action space, and unlike the Simple Robotic Soccer and Pursuit tasks, a condition where $Q$-learning is evaluated without the authors' abstraction will not be evaluated.

\medskip
In the Pursuit and Mario tasks, we use $Q(\lambda)$-learning and $QS(\lambda)$-learning, which are slight variations of the $Q$-learning and $QS$-learning algorithms that use eligibility traces \citep{sutton1998reinforcement}. The addition of eligibility traces to the evaluation was carried out as done by the authors of the recent papers from which the implementations have been taken, namely \cite{brys2014combining} and \cite{suay2016learning}. This allows us to evaluate the different approaches with recently provided baseline solutions without altering their implementations.

\subsection{Experiment 1: Initial Non-Expert Developers Study}\label{sec:exp1}



In this experiment, we seek to investigate the potential benefits of the \textsc{SASS} approach. We focus on technically-able non-experts with some background in programming and RL.  We speculate that participants would find the \textsc{SASS} approach more appealing than the \textsc{FA} approach which in turn will result in designers' $QS$ agents outperforming the designers' $QA$ agents. To examine this hypothesis, we recruited $16$ Computer Science graduate-students majoring in AI - 4 PhD students and 12 Masters students, ranging in age from 23 to 43 (average of 26.8), 10 males and 6 females - to participate in the experiment and act as non-expert designers for two RL agents ($QS$ and $QA$). All participants have some prior knowledge of RL from advanced AI courses (about 2 lectures) yet they cannot be considered experts in the field as they have no significant hands-on experience in developing RL agents. The students are majoring in Machine Learning (7), Robotics (4) and other computational AI sub-fields (5).


We chose to start with the \textit{Simple Robotic Soccer} domain, which is the simplest of the three evaluation domains in this study.    
Prior to the experiment, all subjects participated in an hour-long tutorial reminding them of the basics of $Q$-learning and explaining the Simple Robotic Soccer task's specification. The tutorial was given by the first author of this article, an experienced lecturer and tutor. 
Participants were then given two python codes: First, an implemented $QA$ agent for which participants had to design and implement a state-space abstraction. Specifically, the participants were requested to implement a single function that translates the na\"{\i}ve representation of the state-space to their own state-space representation. Second, participants were given a $QS$ agent for which they had to implement a similarity function. Both codes already implemented all of the needed mechanisms of the game and the learning agents, and they are available at \url{http://www.biu-ai.com/RL}.

In order to allow the participants to evaluate their agent's performance in reasonable time, a basic reward shaping was implemented under both conditions ($QA$ and $QS$) as suggested in the original Simple Robotic Soccer paper \citep{littman1994markov}. The suggested reward shaping is of a Potential Based Reward Shaping (PBRS) structure \citep{ng1999policy}, biasing the player to move towards the goal while on offense and towards the other player while on defense. It is important to note that the use of PBRS allows one to modify the reward function without altering the desired theoretical properties of $Q$-learning and $QS$-learning algorithms. 

We used a within-subjects experimental design where each participant was asked to participate in the task twice, a week apart. In both sessions, the participants' task was to design a learning agent that would outperform a basic $Q$ agent in terms of asymptotic performance and/or average performance (one would suffice to consider the task successful) by using either abstraction or similarities, in no more than 45 minutes of work. 
Ideally, we would want participants to take as much time as they need. However, given that each participant had to dedicate about 3 hours for the experiment (a one hour tutorial, 1.5 hours of programming, and half an hour of logistics) we could not ask participants for more than 45 minutes per condition. Participants were counter-balanced as to which method they were asked to implement first. After each session, subjects were asked to answer a NASA Task Load Index (TLX) questionnaire \citep{TLX}.

In order to ensure the scientific integrity of the submitted agents, participants were requested to perform the task in our lab, in a quiet room, using a designated Linux machine which we prepared for them. 
Furthermore, while programming, a lab assistant (who did not co-author this article) was present to assist with any technical issues. No significant technical difficulties were encountered that might jeopardize the results.  

We then tested the participant's submitted agents against the same hand-coded opponent against whom they had trained. During each session, participants could test the quality of their designed agent at any time by running the testing procedure, which worked as follows: The designed agent was trained for 1,000 games such that after each batch of 50 games, the learning was halted and 10,000 test games were played during which no learning occurred. The winning ratio for these 10,000 test games was presented to the designer after each batch. Given a `reasonable' number of updates per step (i.e., dozens to hundreds), the procedure does not take more than a few seconds on a standard PC. 
In order to allow designers to compare their agents' success to a basic $Q$ agent (the benchmark agent they were requested to outperform), each designer was given a report on a basic $Q$ agent that was trained and tested prior to the experiment using the same procedure described above. After all agents were submitted, each agent was tested and received two scores: one for its average performance during its learning period and one for the asymptotic performance of the agent, i.e., its performance after the training is completed. For this evaluation, we used the same machine used by the study participants, a Linux machine with 16 GB RAM and a CPU with 4 cores, each operating at 4 GHz. 
Each agent was evaluated 50 times over 1,000 episodes, so the score of each episode is in fact an average of the 50 evaluation runs.



\subsubsection*{Results}
Under the $QS$-learning condition, participants defined similarity functions.  
A similarity function is \say{beneficial} only if it helps the $QS$ agent outperform the basic $Q$ agent. Otherwise, we say that the similarity function is \say{flawed} in that it hinders learning.

When analyzing the average performance of the submitted agents, we see that out of the 16 submitted $QS$ agents, 12 (75\%) successfully used a beneficial similarity function. On the other hand, only 3 (19\%) of the 16 $QA$ agents outperformed the $Q$ agent. 
The average winning ratio recorded for the $QS$ agents throughout their training was $68.2\%$, compared to the $42.7\%$ averaged by the $QA$ agent and $60.8\%$ averaged by the benchmark $Q$ agent. 

Asymptotically, 13 out of the 16 $QS$ agents (81\%) outperformed or matched the basic $Q$ agent performance. \textit{None} of the $QA$ agents asymptotically outperformed the $Q$ agent. 
On average, under the $QS$-learning condition, participants designed agents that asymptotically achieved an average winning ratio of 74.5\%. The $QA$ agents achieved only 47.7\% and the $Q$ agent recorded 72.5\%.

Interestingly, \emph{all 16 participants} submitted $QS$ agents which perform better than their submitted $QA$ agents both in terms of average learning performance and asymptotic performance. Namely, the $QS$ agents' advantage over the $QA$ agents is most apparent when examining each designer separately. 
Furthermore, for all participants, the $QS$ agent outperforms the $QA$ agent from the 3$^{rd}$  test (the 150$^{th}$ game) onwards. For 9 of the 16 participants (56\%), the $QS$ agent outperformed the $QA$ agent from the very first test onwards. In addition, the $QS$ agents completed the learning period faster than the $QA$ agents on average, which may imply that a beneficial \textsc{SASS}-based logic is less complex than a \textsc{FA}-based one.

We further analyzed the types of similarities that participants defined under the $QS$-learning condition. This phase was done manually by the authors, examining the participants' codes and trying to reverse-engineer their intentions. Fortunately, due to the task's simple representation and dynamics, distinguishing between the different similarity notions was possible.
It turns out that representational and symmetry similarity notions were the most prevalent among the submitted agents. In 8 of the 16 $QS$ agents (50\%), representational similarities were instantiated, mainly by moving one or both of the virtual players across the grid, assuming that the further away one moves the player(s), the lower the similarity is to the original positioning (See Figure \ref{fig:all}(a)). Symmetry similarities were used by 7 of the 16 participants (43.7\%). All 7 of these agents used the idea of mirroring, where the state and action were mirrored across an imaginary horizontal line dividing the grid in half. Some of them also defined mirroring across an imaginary vertical line dividing the grid in half, with an additional change of switching ball position between the players. While we were able to show that each of these ideas is empirically beneficial on its own, we did not find evidence that combining them  brings about a significant change.  
Transitional similarities were only defined by 2 of the 16 participants (12.5\%). Both of these designers tried to consider a more strategic approach. 
For instance, moving towards the opponent while on defense is considered similar, regardless of the initial position.  It turns out that neither of the  provided transitional similarities were beneficial on their own as they were submitted by the designers.  

Only 4 of the 16 participants (25\%) used more than a single similarity notion while defining the similarity function. Interestingly, the two best performing $QS$ agents combined 2 notions in their similarity function (representational and symmetry similarities). We speculate that combining more than a single similarity notion can be useful for some designers, yet in the interest of keeping with the task's tight time frame, participants refrained from exploring \say{too many different directions} and focused on the ones they initially believed to be the most promising.

Recall that 4 participants (25\%) submitted flawed similarity functions. Although these participants were unable to find a beneficial similarity function, the submitted $QS$ agents were not considerably worse than the basic $Q$-learning. The average performance for these 4 agents was 56.9\% compared to 60.8\% for the basic $Q$ agent, and their average asymptotic score was 61.5\% compared to 72.5\% for the basic $Q$ agent.  

Unlike the significant difference between the $QA$-learning and $QS$-learning conditions in terms of agents' performance, a much larger number of participants is needed to achieve significant results in terms of TLX scores. Using the ANOVA one-way test on the experiment results we find an $f$-ratio of 1.5093 and a $p$-value of 0.2282, which do not reflect a significant difference. The complete TLX results are available at \url{http://www.biu-ai.com/RL}.

Overall, the results are aligned with our initial hypothesis and demonstrate that designers  better utilized the \textsc{SASS} approach compared to the \textsc{FA} approach. The results are summarized in Table \ref{table:exp1_non_experts}.

\begin{table*}[h]
	\centering
	{\def \arraystretch{2}
        \caption{Experiment 1 main results summary}
        \label{table:exp1_non_experts}
        \begin{tabular}{l c c c }
            \toprule
            Criteria &  QS & QA & Q \\ 
            \hline
            Avg. Winning Ratio (during training) & 68.2\% & 42.7\% & 60.8\% \\
            \hline
            Avg. Winning Ratio (asymptotically) & 74.5\% & 47.7\% & 72.5\% \\
            \hline
            Better agent than benchmark (during training) & 75\% & 19\% & - \\
            \hline
            Better agent than benchmark (asymptotically) & 81\% & 0\% & - \\
            \hline
            Best Agent (during training) & 75\% & 0\% & 25\% \\
            \hline
            Best Agent (asymptotically) & 81\% & 0\% & 19\% \\
            \bottomrule
        \end{tabular}
    }
    \par \bigskip
    The main results of Experiment 1 (non-expert study). The results show that the SASS approach allowed most designers to outperform the basic $Q$-learning condition and better infuse their domain-knowledge into the RL agent compared to the FA approach. The higher the score - the  better.
\end{table*}

\subsection{Experiment 2: Non-Expert Developers Study}\label{sec:exp2}

In Experiment 2 we seek to investigate three speedup methods: \textit{FA}, \textit{RS} and \textit{SASS}. Similar to Experiment 1, we speculate that participants would be able to utilize the \textsc{SASS} approach and produce agents which outperform the \textsc{QA} and \textsc{Q} agents. In addition, we speculate that \textsc{RS} would also be successfully utilized by designers to outperform the \textsc{QA} and \textsc{Q} agents. We again focus on technically-able non-expert designers who have a strong background in programming yet a very limited experience with RL. Similar to Experiment 1, we required $32$ human participants, all of whom were senior Bachelors or beginning graduate students who are majoring in AI and have participated in an advanced AI course. The participants ranged in age from 20 to 50 (average of 27.2), 23 male and 9 female. The students are majoring in Machine Learning (22), Robotics (7) and other computational AI sub-fields (3). None of the participants in this experiment participated in Experiment 1. 

Unlike Experiment 1, in this experiment we investigate two more complex RL tasks: Pursuit and Mario. First, we randomly assigned each participant to one of two equally-sized groups. Each group was assigned a different domain; either Pursuit or Mario. Similar to Experiment 1, participants were given three Java codes: an implemented $QA$ agent for which participants had to design and implement a state-space abstraction, a $QS$ agent for which participants had to implement a similarity function, and a $QR$ agent for which participants had to implement a reward shaping function. Note that the last condition ($QR$-learning) was not present in Experiment 1 as a basic shaping reward was already implemented as discussed in Section \ref{sec:exp1}. All codes had already implemented all of the needed mechanisms of the game and the learning agents, and they are available at \url{http://www.biu-ai.com/RL}. 
\textit{It is important to stress that, unlike Experiment 1, we provided no basic reward shaping for the agents.} 

We again use a within-subjects experimental design where each participant was asked to participate in the task \textit{thrice}, with a week separating every two consecutive conditions. Due to the increased complexity of the two domains tested in this experiment compared to Experiment 1, and to allow easy reproducibility of the experiment, participants were given an interactive PowerPoint presentation that introduced the problem domain as well as reminded them of the fundamentals of the tested speedup methods instead of the 1-hour tutorial given in Experiment 1. The PowerPoint presentations are available on our website \url{http://www.biu-ai.com/RL}. As before, in all sessions, the participants' task was to design a learning agent that would outperform a basic $Q$-learning condition in terms of asymptotic performance and/or average performance (one would suffice to consider the task successful) by using either \textsc{FA}, \textsc{SASS}, or \textsc{RS}, in no more than 45 minutes of work for each condition. In this experiment, participants had to devote about 4 hours due to the additional conditions and logistics.

Similar to Experiment 1, participants were instructed to use a designated machine in our lab and were assisted by a lab assistant in case they faced any technical difficulties. No significant technical difficulties were encountered that might jeopardize the results. 

Participants were counter-balanced as to which agent they had to implement first. Following each programming session, the participants were asked to answer a NASA TLX questionnaire. In addition, in order to acquire a better understanding of participants' subjective experience, an additional short questionnaire was administered. The questionnaire consisted of 9 statements to which participants had to rate the degree to which each statement reflects their subjective feeling on a 10-point Likert scale. For instance, \say{To what extent was the speedup method you used appropriate for the task you were required to complete?}. The complete questionnaire is available on our website - \url{http://www.biu-ai.com/RL}. The four key questions, which we will discuss here, can be found in Appendix \ref{app:exp2}.
During each session, participants could test the quality of their agent by running the following testing procedure: In the Pursuit task, the agent trained for 100,000 games, where after each batch of 100 games the average performance of the agent within that batch was presented graphically to the designer. In the Mario task, the agent was trained for 7,500 games, where after each batch of 100 games the average performance of the agent within that batch was presented graphically to the designer. The above procedure is slightly different from Experiment 1 due to time considerations: For the Pursuit task, most submitted agents completed 100,000 training games in no more than a few seconds on a standard PC. On the other hand, for the more complex Mario task, the test procedure took up to half a minute despite the limited training duration of only 7,500 games.

In order to allow designers to compare their agents' success to a basic $Q$-learning condition (the benchmark agent which they were requested to outperform), each designer was given a report on the performance of a basic $Q$ agent that was trained and tested prior to the experiment using the same procedure described above. 

For evaluation, we used the same machine used by the study participants, a Windows machine with 12 GB of RAM and a CPU with eight cores, each operating at 3 GHz.

In addition to the evaluation of the three methods that each designer developed during this experiment, we evaluated an additional condition. We manually combined each of the developers' $QR$ agents with his or her $QS$ agent, resulting in a new agent which we called \textit{QRS} agent. Note that each of the resulting $QRS$ agents uses both the reward shaping and similarity implementations of a specific participant. The $QRS$ agents are similar in spirit to the $QS$ agents from Experiment 1, as reward shaping was also implemented for these agents. 
It is important to mention that participants developed each agent independently and were not informed about this future combination of the $QS$ and $QR$ agents. In total, $128$ agents were evaluated for the two domains combined ($32$ participants, $4$ agents each). 

Recall that in the Pursuit and Mario tasks, we use $Q(\lambda)$-learning and $QS(\lambda)$-learning, which are slight variations of the $Q$-learning and $QS$-learning algorithms that use eligibility traces. 

\subsubsection*{Results}

We report the results for each group separately. 

\underline{Pursuit:} Recall that a submitted agent is considered successful if it outperforms the basic $Q$ agent in at least one of the two criteria of interest: average performance or  asymptotic performance.
In the Pursuit task, the score is the number of steps required by both predators to catch the prey. As a result, it is important to remember that \textit{the lower the score, the better}.

When analyzing the average performance of the submitted agents, we see that out of the 16 submitted $QS$ agents, 14 (87.5\%) successfully used a beneficial similarity function. Similar to the results of Experiment 1, very few of the submitted $QA$ agents (4 out of 16, 25.0\%) were able to outperform the basic $Q$-learning condition. When examining the submitted $QR$ agents, we see similar results to the $QS$-learning condition, with 15 out of the 16 submitted agents (93.75\%) outperforming the $Q$-learning condition. As for the $QRS$ agents, 14 out of the 16 agents (87.5\%) were successful, similar to  the $QS$-learning condition. The $QS$ agents achieved an average training score of 110.07, outperforming the $QA$ agents and the $Q$-learning baseline which scored 131.7 and 143.91, respectively. Interestingly, the $QR$ agents averaged a score of 44.66, less than half of what the $QS$ agents averaged. However, in the $QRS$-learning condition, where we manually combined the $QS$ and $QR$ agents of each study participant, the resulted  agent averaged a score 37.28, reducing the $QR$-learning condition average by 16.5\% and the $QS$-learning condition average by 66\%. 

Evaluating the asymptotic performance of the agents reveals similar results: Out of the 16 $QS$ agents, 12 outperformed or matched the basic $Q$-learning condition performance (75.0\%), averaging 33.03 compared to the asymptotic score of 36.08 by the $Q$ agent. Only 8 of the 16 $QA$ agents (50.0\%) were able to achieve the same, averaging 54.93. Almost all of the submitted $QR$ agents were able to outperform the basic $Q$ agent (15 out of 16, 93.75\%), averaging 25.47. The $QRS$ agents were suited between the $QS$-learning and $QR$-learning conditions, with 13 successful agents out of 16 (81.25\%), averaging 28.5. 

Interestingly, all $QS$, $QR$ and $QRS$ agents that outperformed the $Q$-learning condition on the criteria of average training performance managed to outperform the $Q$-learning condition asymptotically as well. Surprisingly, this does not hold for any of the $QA$ agents. 

The results are summarized in Table \ref{table:exp2_pursuit_non_experts} and illustrated in Figure \ref{fig:ex2_pursuit_non_experts}.

\begin{table*}[h]
	\centering
	{\def \arraystretch{2}
        \caption{Summary of Experiment 2 main results: Pursuit task}
        \label{table:exp2_pursuit_non_experts}
        \begin{tabular}{ p{130pt} c c c c c }
            \toprule
            Criteria & QS & QA & QR & QRS & Q \\
            \hline
            Avg. training performance (turns to win) & 110.07 & 131.70 & 44.66 & \textbf{37.28} & 143.97 \\
            \hline
            Avg. asymptotic performance (turns to win) & 33.03 & 54.93 & \textbf{25.47} & 28.59 & 36.08 \\
            \hline
            Better agent than benchmark (during training) & 14 (87.5\%) & 4 (25.0\%) & \textbf{15 (93.75\%)} & 14 (87.5\%) & - \\
            \hline
            Better agent than benchmark (asymptotically) & 12 (75.0\%) & 8 (50.0\%) & \textbf{15 (93.75\%)} & 13 (81.25\%) & - \\
            \hline
            Overall beneficial agents (during training or asymptotically)& 14 (87.5\%) & 11 (68.75\%) & \textbf{15 (93.75\%)} & 14 (87.5\%) & - \\
            \hline
            Best Agent (during training) & 2 (12.5\%) & 0 (0\%) & \textbf{7 (43.75\%)} & \textbf{7 (43.75\%)} & - \\
            \hline
            Best Agent (asymptotically) & 1 (6.25\%) & 2 (12.5\%) & \textbf{8 (50.0\%)} & 5 (31.25\%) & - \\
            \bottomrule
        \end{tabular}
    }
    \par \bigskip
    The main results of Experiment 2 (non-expert study). The results show that the SASS and RS approaches allow most designers to outperform the basic $Q$-learning condition and better infuse their domain-knowledge into the RL agent compared to the FA approach. The results further show that the $QR$-learning condition consistently outperforms the $QS$-learning condition while the combination of the two, $QRS$-learning, is found to improve the agent's average performance during training for most cases. The lower the score - the better. 
\end{table*}

\begin{figure}[ht]
	\centering 
    \includegraphics[height=8cm]{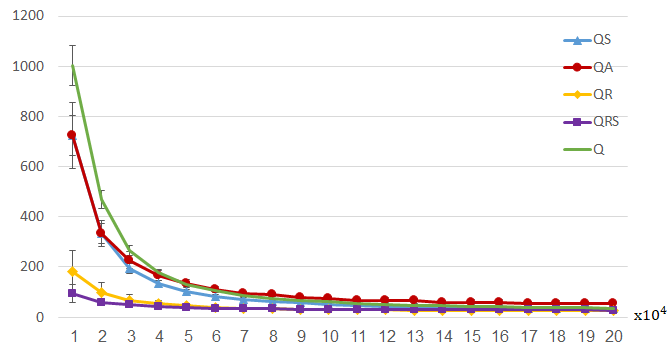}
    \caption{Pursuit agents' average learning curves under the examined conditions. The x-axis marks the number of training games. The y-axis marks the average game score. \textit{The lower the score - the better}. Error bars indicate standard error.} 
\label{fig:ex2_pursuit_non_experts}
\end{figure}

Only 9 out of the 128 agents (7\%) were flawed (2 $QS$ agents, 5 $QA$ agents, a single $QR$ agent and a single $QRS$ agent). In Experiment 1, flawed agents did not perform significantly worse than the baseline $Q$-learning condition. However, in Experiment 2, flawed agents performed quite poorly, scoring an average and asymptotic performance between 4 and 120 times \textit{worse} than the baseline $Q$-learning condition. 

Interestingly, a \textit{strong correlation} was observed between agents' average performance under the $QR$-learning and $QRS$-learning conditions (0.96), whereas a correlation of only 0.22 was found between the $QS$-learning and $QRS$-learning conditions. Very weak \textit{negative} correlations were found between the $QA$ agents' performance and other agents (-0.1 with $QS$ agents and $QR$ agents and -0.12 with the $QRS$ agents). A weak correlation was observed between the $QS$ agents and the $QR$ conditions (0.22). These results suggest that participants who were successful with one method were not necessarily successful with others. The only exception to the above claim is the $QR$-learning condition, which seems to bear the most effect on the $QRS$-learning condition as they are almost perfectly correlated. Results are summarized in Table \ref{table:exp2_pursuit_non_experts_correlation}.

\begin{table*}[t] 
	\centering
  	{\def \arraystretch{2}
    	\caption{Correlation between agent types in Experiment 2: Pursuit task}
        \label{table:exp2_pursuit_non_experts_correlation}
  		\begin{tabular}{ c c c c }
    		\toprule
             & QS & QR & QRS \\
            \hline
            QA          & -0.1012 & -0.0962 & -0.1152 \\
            \hline
            QS          & - & 0.2153 & 0.2262 \\
            \hline
            QR          & - & - & \textbf{0.9564} \\
            \bottomrule
   		\end{tabular}
   }
   \par \bigskip
   A strong positive correlation exists between the number of valid agents in the QR and QRS learning conditions.
\end{table*}

Considering each participant individually, we find that for 7 participants out of 16 (43.8\%) the best-performing agent, in terms of average performance, was the $QRS$ agent. For an additional 7 participants (43.8\%), the best-performing agent was the $QR$ agent. For the remaining 2 participants, the best-performing agent was the $QS$ agent.  Consistent with the results of Experiment 1, the $QA$ agent was not the best-performing agent for any of the participants. Deeper \say{head-to-head} analysis reveals similar trends -- 12 participants developed a $QS$ agent which outperformed their $QA$ agent (75\%) and 13 participants developed $QR$ and $QRS$ agents which outperform their $QS$ agent (81.3\%).
For 8 participants (50\%) the combination of the $QS$ and $QR$ agents -- the $QRS$ agent -- outperformed both their $QR$ and $QS$ agents.

As for the asymptotic performance of the tested agents, we find that for 8 participants out of 16 (50.0\%) the best-performing agent was the $QR$ agent. For an additional 5 participants (31.25\%), the best-performing agent was the $QRS$ agent. For only two participants the best-performing agent was the $QA$ agent and for only a single one the best-performing was the $QS$ agent. Consistent with the above results,  a \say{head-to-head} analysis reveals similar trends -- 11 participants developed a $QS$ agent which outperformed their $QA$ agent (68.75\%) and 13 participants developed $QR$ and $QRS$ agents which outperform their $QA$ agent (81.3\%).
For 10 participants (62.5\%) the combination of the $QS$ and $QR$ agents -- the $QRS$ agent -- outperformed both their $QR$ and $QS$ agents.

We further analyze the types of similarities and reward shaping functions that participants defined under the $QS$-learning and $QR$-learning conditions. This phase was done manually by the authors, examining the participants' codes and attempting to reverse-engineer their intentions. Under the $QS$-learning condition, and contrary to what one may expect, only a single participant instantiated representational similarities. This may be partially attributed to the \say{less-trivial} representation of the state-space (i.e., using differences instead of absolute x,y locations) as implemented in the original paper. Symmetry similarities were used by 5 out of the 16 participants (31.3\%), four of whom used angular rotations with 90$^{\circ}$, 180$^{\circ}$ and 270$^{\circ}$ transpositions of the state around its center (along with the direction of the action, see Figure \ref{fig:all}(b) for an illustration), and 3 of which used mirroring (2 used both). Interestingly, 9 out of 16 participants (56.3\%) defined transitional similarities, considering all state-action pairs which are expected to result in the same state. Under the $QR$-learning condition, most participants (14 out of 16, 87.5\%) developed agents based on motivating the predators to move towards the prey and discouraging them from moving in any other direction. This simple idea was shown to be highly effective, as depicted by the scores discussed above.
The remaining 2 participants also rewarded the predators based on the separation between the predators (intuitively, rewarding the predators for avoiding interfering with each other's moves).  This addition had mixed effects on the agent's performance.  

Considering the participants' TLX scores, using a one-way ANOVA test we find that the scores are significantly different ($F=6.79348$, $p<0.05$). Using post-hoc analysis, we find that the TLX results of both the $QA$-learning and $QR$-learning conditions are \textit{not significantly different}. However, the $QS$-learning condition was found to have higher mean TLX scores compared to both the $QA$-learning and $QR$-learning conditions ($p<0.05$). These results indicate that articulating similarities in the Pursuit domain demands higher levels of developers' effort compared to articulating reward shaping or basic function approximation. 
The full TLX results and tests results can be found on the project's webpage \url{http://www.biu-ai.com/RL}.

In addition to the TLX questionnaire, we administered a customized questionnaire that aims at extracting the participants' subjective experience during the experiment. 
The English version of the questionnaire can be found in Appendix \ref{app:exp2}.
Participants' answers demonstrate a few interesting phenomena: First, participants reported that they understand their task requirements  and purpose well ($Q_1$ in Appendix \ref{app:exp2}, averaging 9 out of 10), with no statistically significant difference between the different conditions. Interestingly, participants reported that the $QS$-learning condition was the most challenging ($Q_2$ in Appendix \ref{app:exp2}, averaging 5.6 compared to 8.1 and 7.5 under the $QA$-learning and $QR$-learning conditions, respectively, $p<0.05$. There was no statistically significant difference between the latter pair.). See Figure \ref{fig:ex2_pursuit_custom_qst} for graphical representation. 
We find support for the above in the participants' TLX scores: The $QS$-learning condition was shown to induce a higher mental demand (averaging 71.25 compared to 59.68 and 44.37 for the $QA$-learning and $QR$-learning conditions, respectively. Here, the difference between the $QA$-learning and $QR$-learning conditions was found to be statistically significant as well, $p<0.05$.). On the other hand, participants reported that under the $QS$-learning condition they could have improved the agent's performance much more if they were to be given more time ($Q_3$ in Appendix \ref{app:exp2}, averaging 6.8 compared to 4.6 and 4.5 under the $QA$-learning and $QR$-learning conditions, respectively, $p<0.05$. There was no statistically significant difference between the latter pair.). This is also supported by participants reporting extremely high time pressure under the $QS$-learning condition, as reflected by  the Temporal Demand index of their TLX scores (averaging 74.37 compared to 43.43 and 38.43 for the $QA$-learning and $QR$-learning conditions, respectively, $p<0.05$. There was no statistically significant difference between the latter pair.). 

\begin{figure}[h]
	\centering \includegraphics[height=4cm,width=0.5\textwidth]{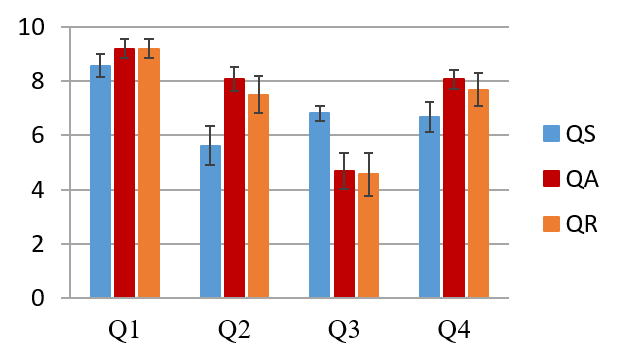}
    \caption{Pursuit post-experiment customized questionnaire average answers. \textit{See Appendix \ref{app:exp2} for details}.}
\label{fig:ex2_pursuit_custom_qst}
\end{figure}

The above results combine to suggest that the $QA$-learning and $QR$-learning methods were more natural for human designers \textit{for the pursuit task}, given the imposed time limit.   
This insight is also aligned with participants reporting the $QS$-learning condition as the least appropriate method for the pursuit task ($Q_4$ in Appendix \ref{app:exp2}, averaging 5.7 compared to 7.3 and 7.1 under the $QA$-learning and $QR$-learning conditions, respectively, $p<0.05$. There was no statistically significant difference between the latter pair). 
The full TLX scores and participants' answers are available at \url{http://www.biu-ai.com/RL}.




\underline{Mario:} As before, a submitted agent is considered successful if it outperforms the basic $Q$ agent in at least one of the two criteria:  average performance or asymptotic performance. 
Unfortunately, the \textbf{vast majority of submitted Mario playing agents were flawed} (71\%).  Specifically, 13 out of the 16 $QS$ agents (81.3\%), 10 out of the 16 $QA$ agents (62.5\%), and 11 out of the 16 $QR$ agents (68.8\%) were flawed. The average learning curves of the different conditions are  illustrated in  Figure \ref{fig:ex2_mario_non_experts}. 

The only significant result in this context is the superiority of the $QS$ agents over the $Q$-learning baseline in the first 4 batches of learning. 
While the $QS$ agents outperform the $QA$, $QR$ and $QRS$ agents, it is important to note that they are all superseded by the baseline $Q$-learning condition on average.


\begin{figure}[ht]
	\centering \includegraphics[height=8cm]{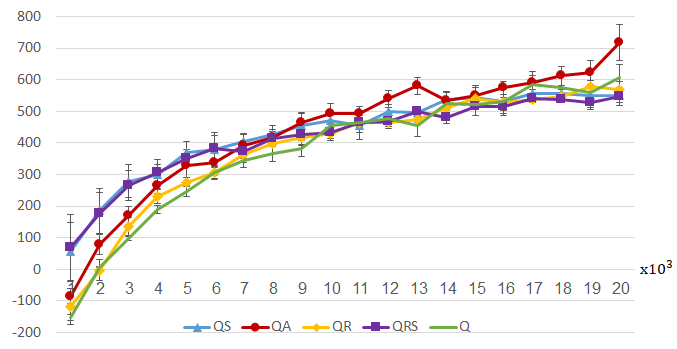}
    \caption{Mario agents' average learning curves under the examined conditions. The x-axis marks the number of training games. The y-axis marks the average game score. \textit{The higher the score - the better}. Error bars indicate standard error.}
\label{fig:ex2_mario_non_experts}
\end{figure}
 
It is uncommon for AI articles to report negative results. Nevertheless, we believe that some useful lessons can be learned from this part of the experiment. Specifically, the answers from participant questionnaires' can shed light on the results. First, participants reported that they understand their task requirements and purpose well ($Q_1$ in Appendix \ref{app:exp2}, averaging 9.5). Thus, a lack of understanding was not the problem in our case. Participants further indicated that they could significantly improve their agent's performance if they were given more time ($Q_3$, averaging 6.2 with no significant differences between the conditions). 
Given the participants' answers, also supported by short, informal interviews we conducted with participants after the experiment, we speculate that the imposed \textit{time constraint} was the main catalyst for  developing flawed agents. It is important to note in this context that the Mario task is significantly more complex than Simple Robotic Soccer or Pursuit in both state-action space and the game dynamics. As a result, participants are likely to require more time to come up with beneficial ideas (taking into account the complex game dynamics) and more time to instantiate different ideas (given the complex state-action space). Moreover, as noted before, Mario's testing procedure took up to half a minute compared to a few seconds in previous tasks. This alone reduced the development time significantly as participants spent a total of a few minutes \say{waiting for results} during their already limited development time. For concreteness, consider the following example: In the Pursuit task, a simple reward shaping function biasing the predators to move closer to the prey performed very well. Here, biasing Mario to move towards the princess (move right) does not work well as Mario has to avoid colliding with enemies, falling into gaps, and he needs to try to collect coins. Designing such a complex reward shaping function and implementing it may take significantly longer than the simplistic one in Pursuit. Furthermore, testing it would take significantly more time.  We find additional support for this hypothesis in Experiment 3 (Section \ref{sec:exp3}), where the human expert designer reported significantly more time needed to develop beneficial agents for the Mario task compared to the Simple Robotic Soccer and Pursuit domains.   

Interestingly, despite the discouraging results described above, participants reported that the $QS$-learning condition was the  \textit{least} challenging  ($Q_2$ in Appendix \ref{app:exp2}, averaging 7.4 compared to 6.1 and 5.8  under the $QA$-learning and $QR$-learning conditions, respectively, $p<0.05$. There was no statistically significant difference between the latter pair.). See Figure \ref{fig:ex2_mario_custom_qst} for graphical representation. 

\begin{figure}[h]
	\centering \includegraphics[height=4cm,width=0.5\textwidth]{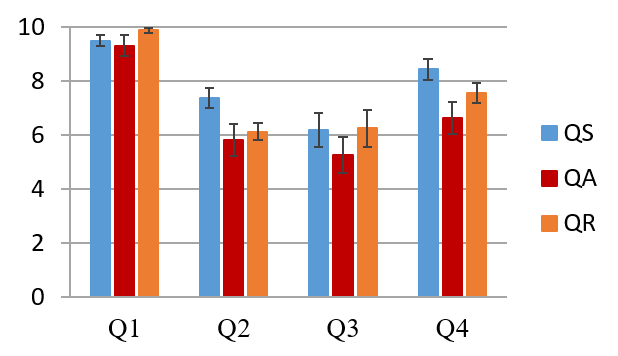}
    \caption{Mario post-experiment customized questionnaire average answers. \textit{See Appendix \ref{app:exp2} for details}}
\label{fig:ex2_mario_custom_qst}
\end{figure}

This is further supported by participants reporting the $QS$-learning condition as the least frustrating in their TLX scores (averaging  40.62 compared to 54.1 and 55.63 for the $QA$-learning and $QR$-learning conditions, respectively, $p<0.05$. There was no statistically significant difference between the latter pair.). 
Moreover,  participants reported the $QS$-learning condition as the \textit{most} appropriate method for the Mario task (averaging 8.3 compared to 6.4 and 4.7 under the $QA$-learning and $QR$-learning conditions, respectively, $p<0.05$. Here, the difference between the $QA$-learning and $QR$-learning conditions was found to be statistically significant as well, $p<0.05$.)

The above results suggest that due to the complexities associated with the Mario task, the time limit was too restrictive. Nevertheless, participants were able to indicate the $QS$-learning condition as the most natural and appropriate technique for this domain. Indeed, it was the only condition that was able to outperform the $Q$-learning on average, yet only for the few first training batches.

When we combine the results for the Pursuit and Mario tasks, they seem to support our initial hypothesis that more $QS$ and $QR$ speedup can allow most designers to produce better performing agents compared to the \textsc{FA} approach. Moreover, the results also seem to imply that the \textsc{RS} speedup method is superior under the Pursuit domain, and its combination with the $QS$ may provide an additional speedup in many cases. 


\subsection{Experiment 3: Expert Developers Study}\label{sec:exp3}

Experiments 1 and 2 focused on non-expert, technically-able human designers. In Experiment 3 we consider RL experts. In this experiment we seek to investigate expert use of the three speedup methods investigated before: \textsc{FA}, \textsc{RS} and \textsc{SASS}. To that end, we recruited 3 highly experienced, expert programmers with a Masters degrees in Computer Science and proven experience in RL (two of whom are 26 years old and the third is 27 years old). None of the experts co-author this paper. Each expert was asked to implement five RL agents: a basic $Q$ agent; a $QS$ agent; a $QA$ agent; a $QR$ agent; and a $QRS$ agent. Each expert was given a single RL task domain: Simple Robotic Soccer, Pursuit or the Mario game, as discussed in Section \ref{sec:domains}. 


Each expert was instructed to take as much time as he needs to implement the agents yet keep track over the invested time for each condition. After all agents were submitted, the second author interviewed each expert about his subjective experience and thoughts during the experiment using a semi-structured interview (see Appendix \ref{app:exp3}). 

Unfortunately, we were unable to get the three experts to come to our lab. As a result, each expert used his own personal computer to program the different agents. The reported running times of the agents are based on our post-hoc evaluation using a personal Linux computer with 16 GB RAM and a CPU with 4 cores, each operating at 4 GHz. 
All technical parameters used by the three experts in this study (learning rates, exploration type, etc.) are fully specified in their codes and are available on the project's webpage  \url{http://www.biu-ai.com/RL}.
For each task, we discuss the implemented agents and their results, followed by the expert's reflections on the task. 




\subsubsection{Simple Robotic Soccer}

For this task, our expert is a 26 year old male who works as a scientific programmer in one of the Israeli Universities. He completed a Masters degree (‘cum laude’) majoring in AI and completed significant works using RL during his Masters and current works.

The expert reported that developing each of the agents required approximately 30 minutes except for the $QRS$ agent, which required only a few minutes given the implemented $QS$ and $QR$ agents.

\noindent
The \underline{$QA$ agent} used a simple distance-based approach, which represented each state according to the learning agent's distance to its opponent and goal.  


\noindent
The \underline{$QS$ agent} used two major similarity notions: First, \textit{representational similarities} -- the agent artificially moves \textit{both players} together across the grid, keeping their original relative distance (see Figure \ref{fig:all}). As the players are moved further and further away from their original positions, the similarity estimation gets exponentially lower. Second, \textit{symmetry similarities} -- experiences in the upper half of the field are mirrored in the bottom part by mirroring states and actions with respect to the $Y$-axis and vice-versa.    
\textit{Transition similarities} were not defined by the expert for this task. 

\noindent
The \underline{$QR$ agent} used a shaping reward similar to the one proposed in \cite{bianchi2014heuristically}. The expert defined that moving towards the goal while on offense and towards the opponent while on defense receives an extra \say{bonus}. Therefore, whenever an action is intended to change the proximity (using the Manhattan distance) to the attacker or the goal (depending on the situation), a PBRS is given.

\noindent
The \underline{$QRS$ agent} combined the main ideas of the $QS$ and $QR$ agents without introducing new ones.

\noindent
\underline{Results:} Each agent was trained for 2,000 games. After each batch of 50 games, the learning was halted and 10,000 test games were played during which no learning occurred. The process was repeated 350 times. 

As expected of an expert, all submitted agents were successful (here they outperformed the baseline $Q$ agent in \textit{both} criteria). 
The results further show that the $QR$ agent outperforms the $QA$ and $QS$ agents from the first batch up to the $19^{th}$ batch, where it is outperformed by the $QS$ agent. Interestingly, the $QRS$ agent seems to take the best of the two, outperforming all agents from the first batch onwards.
See Figure \ref{fig:ex3_soccer_expert} for a graphical representation of the learning curves. 

 
The evaluation of 2,000 games reveals runtime differences between the conditions. The baseline condition, $Q$-learning, runs the fastest, completing the evaluation in 4.5 seconds. A similar runtime was also recorded for the $QR$ agent. The $QA$ agent was a bit slower than the $Q$ agent, requiring about 7 seconds to complete the evaluation. The most time-consuming agents were the $QS$ and $QRS$, requiring about 38-40 seconds each. 


\underline{Expert's Reflections}

The expert developed the $QS$ agent first, based on ideas and thoughts he had while developing the basic $Q$ agent. The implementation of those ideas was non-trivial, so the expert had to depend on \say{trial-and-error} most of the time. 
The $QA$ agent was developed next. The expert claims that this method allowed him to easily translate his knowledge into code. He posits that the \textsc{FA} approach is most similar to the way people evaluate their surroundings before they decide which action to take. He provided the following examples: \say{On a road junction, a driver ignores most of the available information around him and focuses solely on the traffic lights' color in order to decide whether or not to drive forward or keep still. That is what my soccer player did\ldots.}
Then, the $QR$ agent was developed. The expert believed that this was the most time efficient way to accelerate learning and claimed he would use reward shaping as a first speedup tool for future tasks. Following the success of the $QRS$ agent, the expert claimed that the similarity notions should be considered as a \say{second-line} speedup step.

\begin{figure}[h]
	\centering \includegraphics[height=8cm]{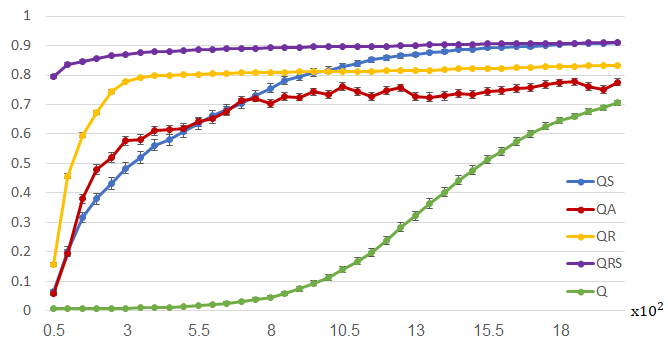}
    \caption{Soccer expert agents' average learning curves under the examined conditions. The x-axis marks the number of training games. The y-axis marks the average game score. The higher the score - the better. Error bars indicate  standard error.}
\label{fig:ex3_soccer_expert}
\end{figure}


\subsubsection{Pursuit}

For this task, our expert is a 27 year old male who has worked as a senior programmer for several years. He completed a Master's degree majoring in AI where his master's project focused on RL.

The expert reported that developing each of the agents required approximately 3 hours, except for the $QRS$ agent which required about 1 hour given the implemented $QS$ and $QR$ agents. 

\noindent
The \underline{$QA$ agent} was already defined by \cite{brys2014combining} who implemented a tile-coding approximation. The expert did not see a reason to change Brys's $QA$-learning implementation. 

\noindent
The \underline{$QS$ agent} was defined based on angular rotations and mirroring. 
Each state is represented as $\langle\Delta_{x_1},\Delta_{y_1},\Delta_{x_2},\Delta_{y_2}\rangle$ where $\Delta_{x_i}$ ($\Delta_{y_i}$) is the difference between predator $i$'s x-index (y-index) and the prey's x-index (y-index), thereby a similarity of $1$ was already set for all states in which the relative positioning of the prey and predators is the same.
Symmetry similarities were defined using 90$^{\circ}$, 180$^{\circ}$ and 270$^{\circ}$ transpositions of the state around its center (along with the direction of the action, see Figure \ref{fig:all}(b)). 
Furthermore, experiences in the upper (left) part of the field are mirrored in the bottom (right) part by mirroring states and actions and vice-versa. 
Transition similarities were defined for all state-action pairs that are expected to result in the same state. 

\noindent
The \underline{$QR$ agent} was designed based on a simple logic that encourages a predator to move towards the prey and punishes moves in any other direction.
The chosen shaping function returned extremely low artificial rewards ($\pm10^{-22}$).

\noindent
The \underline{$QRS$ agent} combined the notion of symmetry from the $QS$-learning condition with the $QR$-learning condition. The use of angular rotations was shown to hinder the $QRS$ agent's performance, and thus those were removed.

\underline{Results}\\ Each agent was trained for 10,000 games. After each batch of 100 games, the learning was halted and 10,000 test games were played during which no learning occurred. 
The process was repeated 50 times. 

Again, as one would expect of an expert, all submitted agents were successful (here they outperformed the baseline $Q$ agent in \textit{both} criteria). 
The results show that the $QRS$ agent is the most efficient one and that it learns significantly faster than other agents. In addition, the $QR$, $QS$ and $QA$ agents outperformed the baseline agent and show large improvements in the convergence rate. 
See Figure \ref{fig:ex3_pursuit_expert} for a graphical representation of the learning process. 


While the $Q$, $QA$ and $QR$ agents complete their training (10,000 games each) in $8.5$ seconds on average (with no significant difference between the two), $QS$ completes the same training in $17.5$ seconds on average. On average, the $QS$ agent updated $12$ entries per iteration. 

\underline{Expert's Reflections}

The expert first implemented the $QA$ agent, followed by the $QS$, $QR$ and $QRS$ agents, in that order. The expert claims that all tested methods were easy to instantiate and implement in the given domain. In terms of design effort, he sees no significant differences between the methods. He points out that under the $QR$-learning condition, the first reward shaping function he tried worked out to be the best one out of approximately a dozen functions he tested. This was not the case for the $QS$ agent, which he mentions to be \say{incremental}, namely the development process was a \say{step-by-step} process where at each step a new similarity notion was introduced, evaluated and refined if needed. He believes that the $QS$ and $QR$ learning conditions are the most intuitive methods he is aware of and he recommends using them, individually or in tandem, on a task-basis. Here, he claims that the $QS$-learning condition was the most appropriate. After we pointed out that the $QR$ agent outperformed the $QS$ agent he revised his answer, deeming both methods as \say{most appropriate}.

\begin{figure}[h]
	\centering \includegraphics[height=8cm]{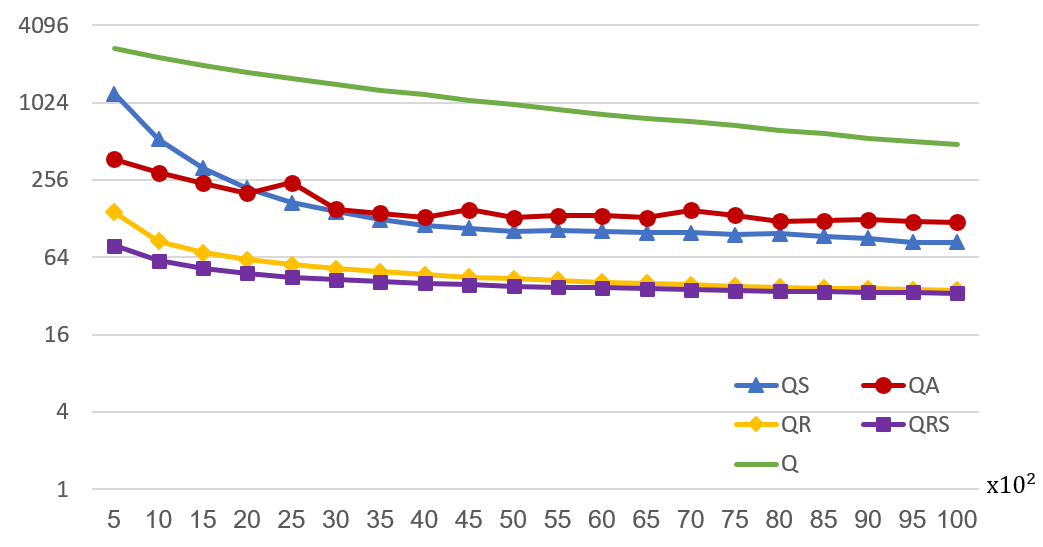}
    \caption{Pursuit expert agents' average learning curves under the examined conditions. The x-axis marks the number of training games. The y-axis marks the average game score \textit{in log-scale}. \textit{The lower the score - the better}. Standard errors are very small and thus are not noted in the figure.}
\label{fig:ex3_pursuit_expert}
\end{figure}


\subsubsection{Mario}

For this task, our expert is a 27 year old programmer who works as a software development team leader at a large international high-tech company. He is completing a Master's degree in Computer Science and has more than 10 years of programming experience, including work with RL.

The expert reported that developing each of the agents required a significant amount of time. The $QS$ agents required about 2.5 hours whereas the $QR$ agent required about 2 hours. The $QRS$ agent required only a few minutes given the implemented $QS$ and $QR$ agents.  

\noindent
The \underline{$QA$ agent} was implicitly defined by \cite{suay2016learning} from which the implementation was taken.  The expert did not change the given abstractions.


\noindent 
The \underline{$QS$ agent} used the following \textit{representational similarity} -- each state representation indicates whether Mario can jump or shoot using 2 Boolean variables. Given a state-action pair in which Mario does not jump or shoot, all respective states (i.e., the four variations of these two Boolean variables) were defined as similar to the original pair. Namely, if Mario walks right, then regardless of Mario's ability to shoot or jump, the state-action pair is considered similar to the original one. 
\textit{Symmetry similarities} were defined using the mirroring of the state-actions across an imaginary horizontal line that divides the environment in half, with Mario in the middle. 
As illustrated in Figure \ref{fig:all}(c), regardless of specific state, performing action $a$ (e.g., move right) is assumed to be similar to using action $a$+``run" (e.g., run right). 

\noindent
The \underline{$QR$ agent} used the following two basic ideas: 1) moving/jumping to the right is better than moving/jumping to the left; 2) avoid getting too close to enemies and obstacles.

The \underline{$QRS$ agent} was a simple combination of the $QS$ and $QR$ agents.

\underline{Results}\\ Each agent was trained for 200,000 games. After each batch of 10,000 games, learning was halted and 1,000 test games were played during which no learning occurred. 
The process was repeated 50 times. 
The two agents are also compared to human performance level as evaluated by \citep{suay2016learning}. 
The results show that the $QS$ agent learns faster than the $QA$ agent. The $QR$ agent learns even faster and outperforms all other agents up to the fifth batch, when it then converges with the $QRS$ agent. Overall, the $QRS$ agent performances very similar to the $QR$ agent but with a slightly worse performance in the first few episodes. 

See Figure \ref{fig:ex3_mario_expert} for a graphical representation of the learning curve.



\underline{Expert's Reflections}

The expert developed the $QS$ agent first. He mentioned that due to the complex state-action space representation, significant time was invested in manipulating encountered state-action pairs in order to generate the desired similar pairs. This task was made somewhat easier when developing the $QR$ agent, not due to the reward shaping technique but rather due to his experience. 
The $QR$ agent took advantage of very basic notions which the expert implemented very fast. The expert claims that the more time he put into developing better reward shaping functions, the \textit{worse} the functions turned out to be. Specifically, his best reward shaping function was the first or second one he tried. Conversely, he mentions that this was not the case for the $QS$-learning condition, in which additional similarities played a useful role in further speeding-up the RL process. He believes that the combination of reward shaping with similarities is the most suitable for this task. 

\begin{figure}[h]
	\centering \includegraphics[height=8cm]{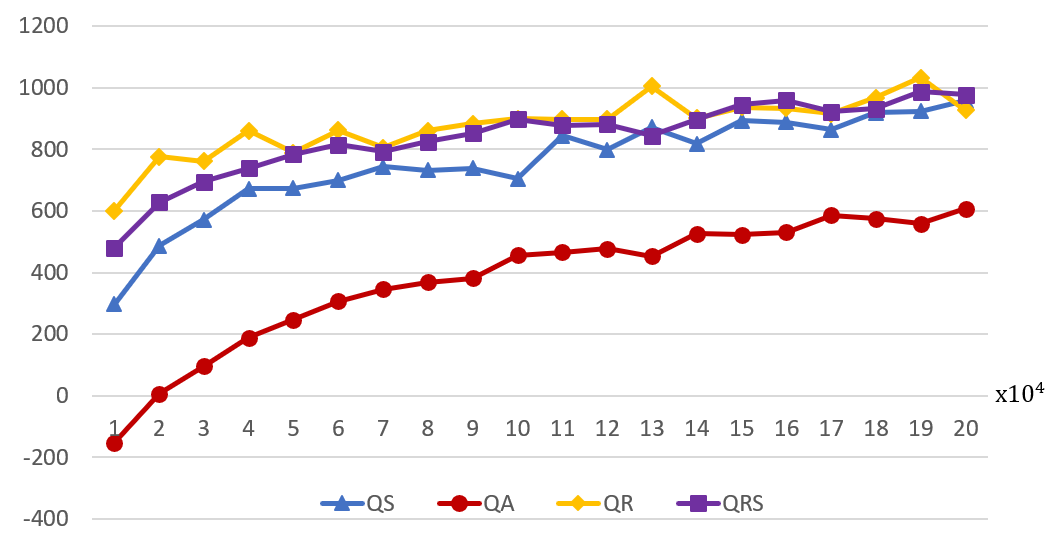}
    \caption{Mario expert agents' average learning curves under the examined conditions. The x-axis marks the number of training games. The y-axis marks the average game score. The higher the score - the better. Standard errors are very small thus are not noted in the figure.}
\label{fig:ex3_mario_expert}
\end{figure}

\section{Conclusions}\label{sec:conclusions}

In this first-of-its-kind human study, we explored how human designers, both expert and non-expert, leverage their knowledge in order to speed up RL. We focused on the challenge of injecting human knowledge into an RL learner using the notions of \textit{abstraction}, \textit{similarity} and \textit{reward shaping}. 

Interestingly, and contrary to its wide popularity in practice, the use of \textit{abstraction} was shown to provide poor speedup results throughout the study. Specifically, in our non-expert experiments (Experiments 1 and 2), the generalization approach (represented by the $QA$ agents) was consistently outperformed by other conditions, and in most cases designers were unable to outperform the baseline $Q$-learning condition using this approach. In our expert experiment (Experiment 3), the results present a similar trend. Specifically, participants were able to use abstraction to improve over the baseline $Q$-learning, yet in all tested settings this condition came last in terms of performance.  Our \textsc{SASS} approach, based on the notion of \textit{similarities} (represented by the $QS$ agents), has demonstrated mixed results. In Experiments 1 and 2, it was shown to outperform the baseline $Q$-learning and abstraction conditions in the vast majority of cases. However, in one of the tasks (Pursuit, Section \ref{sec:chase}) it was also shown to induce high levels of mental and temporal demand. Similar results were recorded in Experiment 3: On the one hand, the proposed method outperformed the $Q$-learning and abstraction conditions. On the other hand, experts disagree on the \say{intuitiveness} of the method. It is also important to note that the method seems to require more time on the designer's part compared to other methods. Quite consistently throughout the study, the \textit{reward shaping} condition (represented using the $QR$ agents) was shown to be both effective and natural for designers. Specifically, in Experiments 2 and 3, participants (both experts and non-experts) reported this technique to be the most suitable and intuitive technique, and in turn it provided superior agent performance compared to the above conditions. An exception is the Mario task in Experiment 2 (Section \ref{sec:exp2}), where all methods performed badly, making the results hard to interpret correctly. 

It turns out that \textbf{the best-performing agents in this study use the combination of reward shaping and similarities} (represented by the $QRS$ agents). In most cases, these agents use a simple (perhaps, na\"ive) combination of the defined similarities under the $QS$-learning conditions with the reward shaping function defined under the $QR$-learning conditions.  This combination is consistently superior to the use of a single speedup method, yet it requires some development overhead since the two methods have to be implemented. Given that the two methods have already been implemented, their combination is usually straightforward and requires negligible time. 

The above results combine to provide another, yet more general, insight: different techniques allow a designer to develop beneficial RL agents. However, the common \say{anecdotal proofs} one is likely to see in RL papers illustrating the usefulness of a proposed technique (usually provided and implemented by the authors themselves) do not guarantee that the technique would be beneficial in practice with other developers and do not provide one with any \say{hint} regarding the potential designers' effort in implementing the proposed approach. We believe that this insight is not restricted to the challenge of injecting human knowledge into an RL learner. Thus, we hope that this work will inspire other researchers to investigate their proposed approaches and techniques in human studies, with actual programmers, to ensure the ecological validity of their contributions.
 
In future work, we plan to extend the proposed  experimental approach to other RL algorithms (e.g., linear function approximation and deep reinforcement learning) and techniques (e.g., learning from demonstrations). As part of this additional step, we further plan to include non-technical users, who are not expected to read or modify code, something which was not included in this study. 

\section*{Acknowledgment}

This article extends our previous reports from AAMAS 2017 \citep{rosenfeld2017speeding} (short paper) and IJCAI 2017 \citep{rosenfeld2017leveraging} (full paper) in several major aspects: First, in the former, the \textsc{SASS} approach was presented and tested by three experts as described in Section \ref{sec:exp3}. Then, in \citep{rosenfeld2017leveraging}, the study was extended to include an additional 16 non-expert designers who implemented the $QS$-learning and $QA$-learning conditions as discussed in Experiment 1 (Section \ref{sec:exp1}). In this article, we almost \textit{triple} our participant pool by recruiting an additional 32 participants and perform an additional experiment (Experiment 2, Section \ref{sec:exp2}). As a result of this addition, we were able to investigate the \textit{reward shaping} condition, which was not investigated in previous reports, and provide a much broader and in-depth investigation of human designers. This addition also enhances the credibility and validity of our previously reported results and demonstrates new insights which were not previously observed.

An extended version of \citep{rosenfeld2017speeding} entitled \textit{\say{Speeding up Tabular Reinforcement Learning Using State-Action Similarities}} was presented at the Fifteenth Adaptive Learning Agents (ALA) workshop at AAMAS 2017 and received the \textit{Best Paper Award} of the workshop.

This research was funded in part by MAFAT. It has also taken place at the Intelligent Robot Learning (IRL) Lab, which is supported in part by NASA NNX16CD07C, NSF IIS-1734558, and USDA 2014-67021-22174.

\appendix
\label{appendix}
\section{Post-experiment subjective evaluation questionnaire (Experiment 2)}
\label{app:exp2}
	\begin{enumerate}
        \item How clear were the task requirements and purpose? 		
        \begin{table*}[h] \centering
          	\begin{tabular}{ c | c | c | c | c | c | c | c | c | c | c | c } 
            	\hline
                Not clear at all &
            	1 &  2 & 3 & 4 & 5 & 6 & 7 & 8 & 9 & 10 & 
                Very clear \\
                \hline
          	\end{tabular}
    	\end{table*}
        
        \item How complex was the task?
        \begin{table*}[h] \centering
          	\begin{tabular}{ c | c | c | c | c | c | c | c | c | c | c | c } 
            	\hline
                Highly complex &
            	1 &  2 & 3 & 4 & 5 & 6 & 7 & 8 & 9 & 10 & 
                Simple \\
                \hline
          	\end{tabular}
    	\end{table*}
        
        \item Say you were given additional time for the task. How much better do you think your agent could have become? 
        \begin{table*}[h] \centering
          	\begin{tabular}{ c | c | c | c | c | c | c | c | c | c | c | c } 
            	\hline
                It would stay the same &
            	1 &  2 & 3 & 4 & 5 & 6 & 7 & 8 & 9 & 10 & 
                Significantly better \\
                \hline
          	\end{tabular}
    	\end{table*}

		\item To what extent do you think that the speedup method you used is appropriate for the task in question?
        \begin{table*}[h] \centering
          	\begin{tabular}{ c | c | c | c | c | c | c | c | c | c | c | c } 
            	\hline
                Not appropriate at all &
            	1 &  2 & 3 & 4 & 5 & 6 & 7 & 8 & 9 & 10 & 
                Very appropriate \\
                \hline
          	\end{tabular}
    	\end{table*}

	\end{enumerate}

\section{Post-experiment semi-structured interview (Experiment 3)}
\label{app:exp3}

\begin{enumerate}
	\item How much effort did you invest while implementing each of the speedup methods?
    \item Which of the speedup methods you used was the most appropriate for speeding up the agent's learning?
    \item What are the advantages and disadvantages of each of the methods?
    \item Which of the methods allowed you to infuse your domain-knowledge to the agent in the most efficient way?
    \item Given a new problem-domain, how would you choose the most appropriate acceleration method to use?
\end{enumerate}

\bibliographystyle{agsm}
\bibliography{KERdoc}

\end{document}